\documentclass[letterpaper]{article} 
\usepackage{aaai23_arxiv}  
\usepackage{times}  
\usepackage{helvet}  
\usepackage{courier}  
\usepackage[hyphens]{url}  
\usepackage{graphicx} 
\usepackage{natbib}  
\usepackage{caption} 
\usepackage{algorithm}

\usepackage{newfloat}
\usepackage{listings}
\DeclareCaptionStyle{ruled}{labelfont=normalfont,labelsep=colon,strut=off} 
\floatstyle{ruled}
\newfloat{listing}{tb}{lst}{}
\floatname{listing}{Listing}
\nocopyright 

\title{Implicit Bilevel Optimization: Differentiating through Bilevel Optimization Programming}

\usepackage{bibentry}

\usepackage{amsmath}
\usepackage{amsfonts}
\usepackage{amssymb}
\usepackage{amsthm}
\usepackage{xcolor}
\usepackage{mathrsfs}
\usepackage{mathtools}
\usepackage{comment}

\usepackage{hyperref} 
\usepackage{subfigure}

\usepackage{widetext}
\usepackage{gensymb}
\usepackage{amsmath}
\usepackage{amsmath,amsfonts,amssymb}
\usepackage{graphicx}
\usepackage{bm}
\usepackage{algorithm,algpseudocode,float}

\usepackage{threeparttable}
\usepackage{multirow}
\usepackage{booktabs}
\usepackage{tablefootnote}
\usepackage{array}
\usepackage{caption}

\newtheorem{thm}{Theorem}

\newcommand{\R}{\mathbb{R}}

\usepackage{mathtools}

\DeclareMathOperator{\E}{\mathbb{E}}

\usepackage{mathtools}

\newcommand{\dd}[1]{\mathrm{d}#1}

\usepackage{wrapfig}
\usepackage{tikz}
\usetikzlibrary{arrows.meta,positioning}

\usepackage{adjustbox}
\usepackage{array}

\usepackage{enumitem}
\newlist{eqlist}{enumerate*}{1}
\setlist[eqlist]{itemjoin=\quad,mode=unboxed,label=(\roman*),ref=\theequation(\roman*)}

\usepackage{enumitem}

\usepackage[normalem]{ulem}

\usetikzlibrary{backgrounds}

\usepackage[para]{footmisc}

\usepackage{dblfloatfix}

\usepackage{enumitem}
\setlist{nolistsep}

\makeatletter
\usepackage{comment}
\let\wfs@comment@comment\comment
\let\comment\@undefined
\usepackage{changes}
\let\wfs@changes@comment\comment
\let\comment\@undefined
\newcommand\comment{%
    \ifthenelse{\equal{\@currenvir}{comment}}
    {\wfs@comment@comment}
    {\wfs@changes@comment}%
}
\makeatother

\usepackage{xspace}
\newcommand{\bigrad}{\textsc{BiGrad}\@\xspace}

\usepackage[american]{babel}

\usepackage{mathtools} 
\usepackage{booktabs} 
\usepackage{tikz} 

\author {
    Francesco Alesiani
}
\affiliations{
NEC Laboratories Europe, 
    Heidelberg, 
    Germany \\
    \href{mailto:Francesco.Alesiani@neclab.eu}{\texttt{Francesco.Alesiani@neclab.eu}}
}

\begin{document}
%
%
%
%
%
%
\maketitle              

\begin{abstract}
Bilevel Optimization Programming is used to model complex and conflicting interactions between agents, for example in Robust AI or Privacy-preserving AI. Integrating bilevel mathematical programming within deep learning is thus an essential objective for the Machine Learning community. 
Previously proposed approaches only consider single-level programming. In this paper, we extend existing single-level optimization programming approaches and thus propose {\it Differentiating through Bilevel Optimization Programming} 
(\bigrad
) 
for end-to-end learning of models that use Bilevel Programming as a layer. 
\bigrad has wide applicability and can be used in modern machine learning frameworks. \bigrad is applicable to both continuous and combinatorial Bilevel optimization problems. We describe a class of gradient estimators for the combinatorial case which reduces the requirements in terms of computation complexity; for the case of the continuous variable, the gradient computation takes advantage of the push-back approach (i.e. vector-jacobian product) for an efficient implementation. Experiments show that the \bigrad successfully extends existing single-level approaches to Bilevel Programming.
\end{abstract}


\section{Introduction}\label{sec:intro}

Neural networks provide unprecedented improvements in perception tasks, however, deep neural networks do not natively protect against adversarial attacks nor preserve the privacy of the training dataset. In recent years various approaches have been proposed to overcome this limitation \citep{shafique2020robust}, for example by integrating adversarial training \cite{xiao2020adversarial}.  Some of these approaches require solving some optimization problems during training. 
Recent approaches propose thus differentiable layers that incorporate either quadratic \citep{amos2017optnet}, convex \citep{agrawal2019differentiable}, cone  \citep{agrawal2019differentiating}, equilibrium  \citep{bai2019deep}, SAT \citep{wang2019satnet} or combinatorial \citep{poganvcic2019differentiation,mandi2020interior,berthet2020learning} programs. The use of optimization programming as a layer of differentiable systems requires computing the gradients through these layers. With discrete variables, the gradient is zero almost everywhere, while with complex (black box) solvers, the gradient may not be accessible. 
\begin{figure}
	\centering
	\includegraphics[width=0.45\textwidth,trim=0cm 0cm 0cm 0cm, clip]{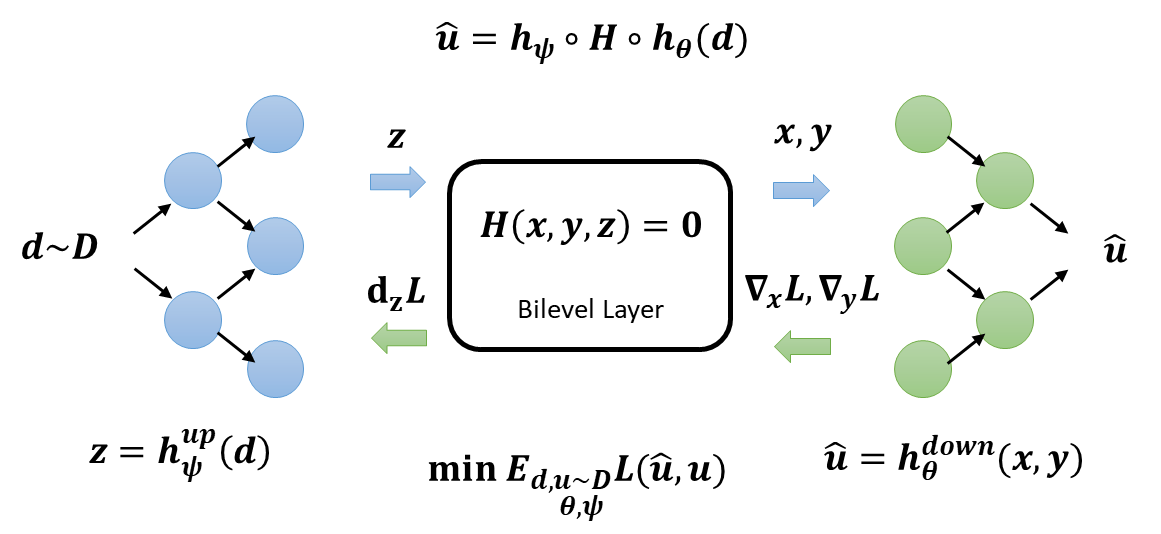}
	\caption{The Forward and backward passes of a Bilevel Programming (\bigrad) layer: the larger system has input $d$ and output $u = h_\psi \circ H  \circ h_\theta (d)$; the bilevel layer has input $z$ and output $x,y$, which are solutions of a Bilevel optimization problem represented by the implicit function $H(x,y,z)=0$.}
	\label{fig:implicit_layer}
\end{figure}

Proposed gradient estimates either relax the combinatorial problem \citep{mandi2020interior}, perturb the input variables \citep{berthet2020learning,domke2010implicit} or linearly approximate the loss function \citep{poganvcic2019differentiation}.
These approaches though, do now allow to directly express models with conflicting objectives, for example in structural learning \cite{elsken2019neural} or adversarial system \cite{goodfellow2014generative}.
We thus consider the use of bilevel optimization programming as a layer.
Bilevel Optimization Program 
\citep{kleinert2021survey,dempe2018bilevel}, 
also known as a generalization of Stackelberg Games, is the extension of a single-level optimization program, where the solution of one optimization problem (i.e. the outer problem) depends on the solution of another optimization problem (i.e. the inner problem). This class of problems can model interactions between two actors, where the action of the first depends on the knowledge of the counter-action of the second. 
Bilevel Programming finds application in various domains, as in Electricity networks, Economics, Environmental policy, Chemical plants, defense, and planning 
\citep{dempe2018bilevel}. 
We introduce at the end of the section example applications of Bilevel Optimization Programming. 

In general, Bilevel programs are NP-hard 
\citep{dempe2018bilevel},
they require specialized solvers and it is not clear how to extend single-level approaches since the standard chain rule is not directly applicable.
By modeling the bilevel optimization problem as an implicit layer \citep{bai2019deep}, we consider the more general case where 1) the solution of the bilevel problem is computed by a bilevel solver; thus leveraging on powerfully solver developed over various decades \citep{kleinert2021survey}; and 2) the computation of the gradient is more efficient since we do not have to propagate gradient through the solver. 

We thus propose Differentiating through Bilevel Optimization Programming (\bigrad):
\begin{itemize}
\item \bigrad (\autoref{sec:bigrad}) comprises of forwarding pass, where existing solvers (e.g. \citep{yang2021provably}) can be used, and backward pass, where \bigrad estimates gradient for both continuous (\autoref{sec:continous-problem}, \autoref{sec:continuous}) and combinatorial (\autoref{sec:combinatorial-problem},\autoref{sec:discrete}) problems based on sensitivity analysis; 
\item we show how the proposed gradient estimators relate to the single-level analogous and that 
the proposed approach is beneficial in both continuous (\autoref{sec:OptimalControl}) and combinatorial optimization (\autoref{sec:Robust},\autoref{sec:SP},\autoref{sec:TSP}, ) learning tasks. 
\end{itemize}
\vspace{-.2cm}
\subsubsection{Adversarial attack in Machine Learning}
Bilevel programming is used the represents the interaction between a machine learning model ($y$) and a potential attacker ($x$) \cite{goldblum2019adversarially} and is used to increase the resilience to intentional or unintended adversarial attacks.
\vspace{-.2cm}
\subsubsection{Min-max problems}
Min-max problems are used to model robust optimization problems \citep{ben2009robust}, where a second variable represents the environment and is constrained to an uncertain set that captures the unknown variability of the environment. 
\vspace{-.2cm}
\subsubsection{Closed-loop control of physical systems}
Bilevel Programming is able to model the interaction of a dynamical system ($x$) and its control sub-system ($y$), as, for example, of an industrial plant or a physical process.
The control sub-system changes based on the state of the underlying dynamical system, which itself solves a physics constraint optimization problem 
\citep{de2018end}. 
\vspace{-.2cm}
\subsubsection{Interdiction problems}
Two actors' discrete Interdiction problems \citep{fischetti2019interdiction} arise when one actor ($x$) tries to interdict the actions of another actor ($y$) under budget constraints. These problems can be found in marketing, protecting critical infrastructure, and preventing drug smuggling to hinder nuclear weapon proliferation.

\section{Differentiable Bilevel Optimization Layer}
We model the Bilevel Optimization Program as an Implicit Layer \citep{bai2019deep}, i.e. as the solution of an implicit equation $H(x,y,z)=0$. We thus compute the gradient using the implicit function theorem, where $z$ is given and represents the parameters of our system we want to estimate, and $x,y$ are output variables (Fig.\ref{fig:implicit_layer}). We also assume we have access
to a bilevel solver $(x,y) = \text{Solve}_H (z)$, e.g. \citep{yang2021provably}.
The bilevel Optimization Program is then used as layer of a differentiable system, whose input is $d$ and output is given by $u=h_\psi \circ \text{Solve}_H \circ h_\theta (d)=h_{\psi,\theta}(d)$,  
where $ \circ$ is the function composition operator. We want to learn the parameters $\psi,\theta$ of the function $h_{\psi,\theta}(d)$ that minimize the loss function $L(h_{\psi,\theta}(d),u)$, using the training data $D^\text{tr}=\{(d,u)_{i=1}^{N^{\text{tr}}}\}$.
In order to be able to perform the end-to-end training, we need to back-propagate the gradient $\dd_z L$ of the Bilevel Optimization Program Layer, which can not be accomplished only using the chain rule.

\subsection{Continuous Bilevel Programming} \label{sec:continous-problem}
We now present the definition of the continuous Bilevel Optimization problem, which comprises two non-linear functions $f,g$, as 
	\begin{align} \label{eq:bilevel_continous}
		\min_{x \in X} & f(x,y,z) ~~&
		y \in &\arg \min_{y \in Y} g(x,y,z)
	\end{align}
where the left part problem is called {\it outer optimization problem} and resolves
for the variable $x \in X$, with $X=\R^n$. The right problem is called the {\it inner optimization problem } and solves for the variable $y \in Y$, with $Y=\R^m$. The variable $z \in \R^p$ is the input variable and is a parameter for the bilevel problem. 
Min-max is a special case of Bilevel optimization problem
	$\min_{y \in Y} \max_{x \in X} g(x,y,z)$,
where the minimization functions are equal and opposite in sign.  In Sec.\ref{sec:linear_equality_and_nonlinear_inequality}, we describe how the model of Eq.~\ref{eq:bilevel_continous} can be extended in the case of linear and nonlinear constraints.
\subsection{Combinatorial Bilevel Programming} \label{sec:combinatorial-problem}
When the variables are discrete, we restrict the objective functions to be multi-linear \citep{Greub_1967}. Various important combinatorial problems are linear in discrete variables (e.g. VRP, TSP, SAT \footnote{Vehicle Routing Problem, Boolean satisfiability problem.}), one example form is the following
	\begin{align} \label{eq:bilevel_discrete}
		\min_{x \in X} \langle z,x \rangle_A + \langle  y,x \rangle_B,  ~~
		y  \in \arg \min_{y \in Y} \langle  w,y\rangle_C + \langle  x,y\rangle_D
	\end{align}
The variables $x,y$ have domains in $x \in X, y \in Y$, where $X,Y$ are convex polytopes that are constructed from a set of distinct points $\mathcal{X} \subset \R^n, \mathcal{Y} \subset \R^m,$ as their convex hull.  The outer and inner problems are Integer Linear Programs (ILPs). 
The multi-linear operator is represented by the inner product $\langle  x,y\rangle_A = x^TAy$
. We only consider the case where we have separate parameters for the outer and inner problems, $z \in \R^p$ and $w \in \R^q$. 



\section{\bigrad: Gradient estimation} \label{sec:bigrad}
\bigrad provides gradient estimations for both continuous and discrete problems. 
We can identify the following common basic steps (Alg.\ref{alg:BIL}):
\begin{enumerate}
    \item In the forward pass, solve the combinatorial or continuous Bilevel Optimisation problem as defined in Eq.\ref{eq:bilevel_continous}(or Eq.\ref{eq:bilevel_discrete}) using existing solver ($\text{Solve}_H (z)$) e.g. \citep{yang2021provably};
    \item During the backward pass, compute the gradient $\dd_z L$ (and $\dd_w L$) using the suggested gradients (Sec.\ref{sec:continuous} and Sec.\ref{sec:discrete}) starting from the gradients on the output variables $\nabla_x L$ and $\nabla_y L$.
\end{enumerate}

\begin{algorithm}
	\begin{enumerate}
	    \item {\bf Input}: Training sample $(\tilde{d},\tilde{u})$ \;
	\item {\bf Forward Pass}: \;
	\begin{enumerate}
	    \item Compute $(x,y) \in \{x,y : H(x,y,z) = 0\}$ using Bilevel Solver:  $(x,y) \in \text{Solve}_H (z) $\;
	    \item Compute the loss function
	    $L(h_\psi \circ H  \circ h_\theta (\tilde{d}),\tilde{u})$,
	    \item Save $(x,y,z)$ for the backward pass
	\end{enumerate}
	\item {\bf Backward Pass}: \;
	\begin{enumerate}
	    \item updates the parameter of the downstream layers $\psi$ using back-propagation \;
	    \item For the continuous variable case, compute based on Theorem~\ref{th:bigrad_cont} around the current solution $(x,y,z)$, without solving the Bilevel Problem
	    \item For the discrete variable case, use the gradient estimates of Theorem~\ref{th:discrete} or 
	    Section \ref{sec:discrete} (e.g. Eq.\ref{eq:discrete_implicit_single_merged} or Eq.\ref{eq:discrete_through}) 
	    by solving, when needed, the two separate problems\;
	    \item Back-propagate the estimated gradient to the downstream parameters $\theta$
	\end{enumerate}	
	\end{enumerate}
	\vspace{4mm}
	\caption{\bigrad Layer: Bilevel Optimization Programming Layer using \bigrad }
	\label{alg:BIL}
\end{algorithm}

\subsection{Continuous Optimization gradient estimation}
 \label{sec:continuous}
To evaluate the gradient of the variables $z$ versus the loss function $L$, we need to propagate the gradients of the two output variables $x,y$ through the two optimization problems. We can use the implicit function theorem to approximate locally the function $z \to (x,y)$. We have thus the following main results\footnote{Proofs are in the Supplementary Material}.
\begin{thm}\label{th:items}
	Considering the bilevel problem of Eq.\ref{eq:bilevel_continous}, we can build the following set of equations that represent the equivalent problem around a given solution $x^*,y^*,z^*$:
		\begin{align}\label{eq:bilevel_continous_eq}
			F(x,y,z) &= 0 ~~& 
			G(x,y,z) &= 0
		\end{align}
	where
		\begin{align} \label{eq:bilevel_continous_items}
			F(x,y,z) &= \nabla_x f- \nabla_y f \nabla_y G  \nabla_x G,   ~ &
			G(x,y,z) &= \nabla_y g
		\end{align}
	where we used the short notation $f=f(x,y,z),g=g(x,y,z),F=F(x,y,z), G=G(x,y,z)$
\end{thm}
\begin{thm} \label{th:bigrad_cont}
Consider the problem defined in Eq.\ref{eq:bilevel_continous}, then the total gradient of the parameter $z$ w.r.t. the loss function $L(x,y,z)$ is computed from the partial gradients $\nabla_x L, \nabla_y L, \nabla_z L$ as 
\begin{align} \label{eq:bigrad_continuous}
	\dd_z L &= \nabla_z L - 
	\begin{vmatrix}
	\nabla_x L  & \nabla_y L
	\end{vmatrix}
	\begin{vmatrix}
		\nabla_x F & \nabla_y F\\
		\nabla_x G & \nabla_y G
	\end{vmatrix}^{-1}
	\begin{vmatrix}
	\nabla_z F  \\
	 \nabla_z G
\end{vmatrix}
\end{align}
\end{thm}
The implicit layer is thus defined by the two conditions $F(x,y,z)=0$ and $G(x,y,z)=0$. We notice that Eq.\ref{eq:bigrad_continuous} can be solved without explicitly computing the Jacobian matrices and inverting the system, but by adopting the Vector-Jacobian product approach we can proceed from left to right to evaluate $\dd_z L$. In the following section, we describe how affine equality constraints and nonlinear inequality can be used when modeling $f,g$. We also notice that the solution of Eq.\ref{eq:bigrad_continuous} does not require solving the original problem, but only applying matrix-vector products, i.e. linear algebra, and the evaluation of the gradient that can be computed using automatic differentiation. 
The extension of Theorem.\ref{th:bigrad_cont} to cone programming is presented in Sec.\ref{sec:bilevel_cone}.  

\subsection{Combinatorial Optimization gradient estimation}\label{sec:discrete}
When we consider discrete variables, the gradient is zero almost everywhere. 
We thus need to resort to estimating gradients. For the bilevel problem with discrete variables of Eq.\ref{eq:bilevel_discrete}, when the solution of the bilevel problem exists and its solution is given \citep{kleinert2021survey}, Thm.\ref{th:discrete} gives the gradients of the loss function with respect to the input parameters.
\begin{thm}\label{th:discrete}
Given the Eq.\ref{eq:bilevel_discrete} problem, the partial variation of a cost function $L(x,y,z,w)$ on the input parameters has the following form:
\begin{subequations}\label{eq:discrete_partial_grad}
\begin{align}
\dd_z L &= \nabla_z L + [\nabla_x L + \nabla_y L \nabla_x y] \nabla_z x \\
\dd_w L &= \nabla_w L + [\nabla_x L \nabla_y x + \nabla_y L] \nabla_w y 
\end{align}
\end{subequations}
\end{thm}
The $ \nabla_x y, \nabla_y x$ terms capture the interaction between outer and inner problems. We could estimate the gradients in Thm.\ref{th:discrete} using the perturbation approach suggested in \citep{berthet2020learning}, which estimates the gradient as the expected value of the gradient of the problem after perturbing the input variable, but, similar to REINFORCE \citep{williams1992simple}, this introduces large variance.
While it is possible to reduce variance in some cases \citep{grathwohl2017backpropagation} with the use of additional trainable functions, we consider alternative approaches as described in the following. 

\subsubsection{Differentiation of  black box  combinatorial  solvers} \label{sec:implicit}
\citep{poganvcic2019differentiation} propose a way to propagate the gradient through a single-level combinatorial solver, where  
$\nabla_z L  \approx \frac1{\tau} [ x( z + \tau \nabla_x L) - x(z)]$ when $x(z) = \arg \max_{x \in X} \langle x,z \rangle$. 
We thus propose to compute the variation on the input variables from the two separate problems of the Bilevel Problem:
\begin{subequations}\label{eq:discrete_implicit}
\begin{align}
\nabla_z L  &\approx 1/{\tau} [ x( z + \tau A\nabla_x L,y) - x(z,y)] ~~ \\
\nabla_w L  &\approx  1/{\tau} [ y( w + \tau C \nabla_y L,x) - y(w,x)] 
\end{align}
\end{subequations}
or alternatively, if we have only access to the Bilevel solver and not to the separate ILP solvers, we can express 
\begin{align}\label{eq:discrete_implicit_single_merged}
\nabla_{z,w} L  &\approx 
 1/{\tau} [ s( v + \tau E\nabla_{x,y} L) - s(v)]
\end{align}
where $x(z,y)$ and $y(w,x)$ represent the solutions of the two problems separately, $s(v) = (z,w) \to (x,y)$ the complete solution to the Bilevel Problem, $\tau \to 0$ is a hyper-parameter and $E = \begin{bmatrix} A &0 \\0 &C \end{bmatrix}$. This form is more convenient than Eq.\ref{eq:discrete_partial_grad} since it does not require computing the cross terms, ignoring thus the interaction of the two levels.   


\subsubsection{Straight-Through gradient}\label{sec:losses}
In estimating the input variables $z,w$ of our model, we may not be interested in the interaction between the two variables $x,y$. 
Let us consider, for example, the squared $\ell_2$  loss function defined over the output variables
$$
L^2(x,y) = L^2(x) + L^2(y) 
$$
where $L^2(x)= \frac1{2} \| x-x^*\|^2_2$ and $x^*$ is the true value. The loss is non-zero only when the two vectors disagree, and with integer variables, it counts the difference squared, or, in the case of the binary variables, it counts the number of differences.
If we compute $\nabla_x  L^2(x)= (x - x^*)$ in the binary case, we have that $\nabla_{x_i}  L^2(x) = +1$ if $ x^*_i=0 \land x_i=1$, $\nabla_{x_i}  L^2(x) = -1$ if $ x^*_i=1 \land x_i=0$, and $0$ otherwise. This information can be directly used to update the $z_i$ variable in the linear term $\langle z,x \rangle$, thus we can estimate the gradients of the input variables as $\nabla_{z_i}L^2 = - \lambda \nabla_{x_i}L^2$ and $\nabla_{w_i}L^2 = - \lambda \nabla_{y_i}L^2$, with some weight $\lambda>0$. The intuition is that the weight $z_i$ associated with the variable $x_i$ is increased when the value of the variable $x_i$ reduces. In the general multilinear case, we have additional multiplicative terms. Following this intuition 
(see Sec.A.3), 
we thus use as an estimate of the gradient of the variables 
\begin{align}\label{eq:discrete_through}
\nabla_z L  &= - A \nabla_x L ~~&
\nabla_w L  &= - C \nabla_y L 
\end{align}
This is equivalent in Eq.\ref{eq:bilevel_discrete} where $\nabla_z x = \nabla_w y  = -I$ and $\nabla_y x = 0$, thus $\nabla_x y = 0$. This update is also equivalent to Eq.\ref{eq:discrete_implicit}, without the solution computation. The advantage of this form is that it does not require solving for an additional solution in the backward pass. For the single-level problem, the gradient has the same form as the Straight-Through gradient proposed by \citep{bengio2013estimating}, with surrogate gradient $\nabla_z x = -I$.

\section{Related Work}
\paragraph{Bilevel Programming in machine learning}
Various papers model machine learning problems as Bilevel problems, for example in Hyper-parameter Optimization \citep{mackay2019self,franceschi2018bilevel}, Meta-Feature Learning \citep{li2016learning}, Meta-Initialization Learning \citep{rajeswaran2019meta}, Neural Architecture Search \citep{liu2018darts}, Adversarial Learning \citep{li2019learning}
and Multi-Task Learning \citep{alesiani2020towards}. In these works, the main focus is to compute the solution to the bilevel optimization problems. In  \citep{mackay2019self,lorraine2018stochastic}, the best response function is modeled as a neural network and the solution is found using iterative minimization, without attempting to estimate the complete gradient. Many bilevel approaches rely on the use of the implicit function to compute the hyper-gradient (Sec.~3.5 of \citep{colson2007overview}) but do not use bilevel as a layer.
\paragraph{Quadratic, Cone and Convex {single-level} Programming}
Various works have addressed the problem of differentiate through quadratic, convex, or cone programming \citep{amos2019differentiable,amos2017optnet,agrawal2019differentiating,agrawal2019differentiable}. In these approaches, the optimization layer is modeled as an implicit layer and for the cone/convex case, the normalized residual map is used to propagate the gradients. Contrary to our approach, this work only addresses single-level problems. These approaches do not consider combinatorial optimization.
\paragraph{Implicit layer Networks} 
While classical deep neural networks perform a single pass through the network at inference time, a new class of systems performs inference by solving an optimization problem. Examples of this are Deep Equilibrium Network (DEQ) \citep{bai2019deep} and NeurolODE (NODE) \citep{chen2018neural}. Similar to our approach, the gradient is computed based on a sensitivity analysis of the current solution. These methods only consider continuous optimization. 

\paragraph {Combinatorial Optimization (CO)}
Various papers estimate gradients of single-level combinatorial problems using relaxation. \citep{wilder2019melding,elmachtoub2017smart,ferber2020mipaal, mandi2020interior} for example use $\ell_1,\ell_2$ or log barrier to relax the Integer Linear Programming (ILP) problem. Once relaxed the problem is solved using standard methods for continuous variable optimization. 
An alternative approach is suggested in other papers. For example, in \citep{poganvcic2019differentiation} the loss function is approximated with a linear function and this leads to an estimate of the gradient of the input variable similar to the implicit differentiation by perturbation form \citep{domke2010implicit}. \citep{berthet2020learning} is another approach that uses also perturbation and change of variables to estimate the gradient in an ILP problem.  SatNet \citep{wang2019satnet} solves MAXSAT problems by solving a continuous semidefinite program (SDP) relaxation of the original problem. These works only consider single-level problems.

\paragraph{Discrete latent variables} 
Discrete random variables provide an effective way to model multi-modal distributions over discrete values, which can be used in various machine learning problems.
Gradients of discrete distribution are not mathematically defined, thus, in order to use the gradient-based method, gradient estimations have been proposed. A class of methods is based on the Gumbel-Softmax estimator 
\citep{maddison2016concrete}. 
Gradient estimation of the exponential family of distributions over discrete variables is estimated using the perturb-and-MAP method in \citep{niepert2021implicit}.

\paragraph{Predict then optimize}
Predict then Optimize (two-stage) \citep{elmachtoub2017smart,ferber2020mipaal} or solving linear programs and submodular maximization from \citep{wilder2019melding} solve optimization problems when the cost variable or the minimization function is directly observable. On the contrary, in our approach we only have access to a loss function on the output of the bilevel problem, thus allowing us to use it as a layer. 

\paragraph{Neural Combinatorial Optimization (NCO)}
NCO employs deep neural networks to  derive efficient CO heuristics. NCO includes supervised learning \citep{joshi2019efficient}  and reinforcement learning \citep{kool2018attention}. 


\section{Experiments}
We evaluate \bigrad with continuous and combinatorial problems to show that improves over single-level approaches. In the first experiment, we compare the use of \bigrad versus the use of the implicit layer proposed in \citep{amos2017optnet} for the design of Optimal Control with adversarial noise. 
In the second part, after experimenting with an adversarial attack, we explore the performance of \bigrad with two combinatorial problems with Interdiction, where we adapted the experimental setup proposed in \citep{poganvcic2019differentiation}. In these latter experiments, we compare the formulation in Eq.\ref{eq:discrete_implicit_single_merged} (denoted by Bigrad(BB)) and the formulation of Eq.\ref{eq:discrete_through} (denoted by Bigrad(PT)). In addition, we compare with the single level BB-1 from \citep{poganvcic2019differentiation} and single level straight-through \citep{bengio2013estimating,Paulus_Maddison_Krause_2021}, with the surrogate gradient $\nabla_z x = -I$, (PT-1) gradient estimations. We compare against Supervised learning (SL), which ignores the underlying structure of the problem and directly predicts the solution of the bilevel problem.  

\subsection{Optimal Control with adversarial disturbance}\label{sec:OptimalControl}
We consider the design of robust stochastic control for a Dynamical System \citep{agrawal2019differentiating}. The problem is to find a feedback function $u = \phi(x)$ that minimizes
\begin{subequations}\label{eq:optimal_control_main}
\begin{align} 
\min_\phi & \E \frac1{T} \sum_{t=0}^{T} \| x_t\|^2 + \| \phi(x_t)\|^2 ~~  \\ 
\text{s.t.} ~& x_{t+1} = A x_t + B \phi(x_t) + w_t, \forall t
\end{align}
\end{subequations}
where $x_t \in \R^n$ is the state of the system, while $w_t$ is a i.i.d. random disturbance and $x_0$ is given initial state.
\begin{figure}[]
	\centering
	\subfigure[] {
    \includegraphics[width=0.2\textwidth, trim = 0 0 0 .1cm,clip]{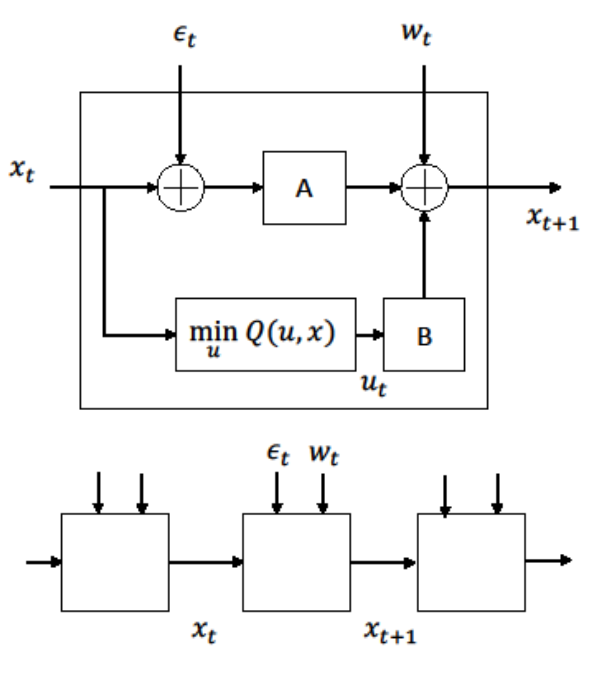}
	}
	\subfigure[] {\centering	\includegraphics[width=0.2\textwidth, trim = .1cm .2cm 1.5cm .1cm, clip]{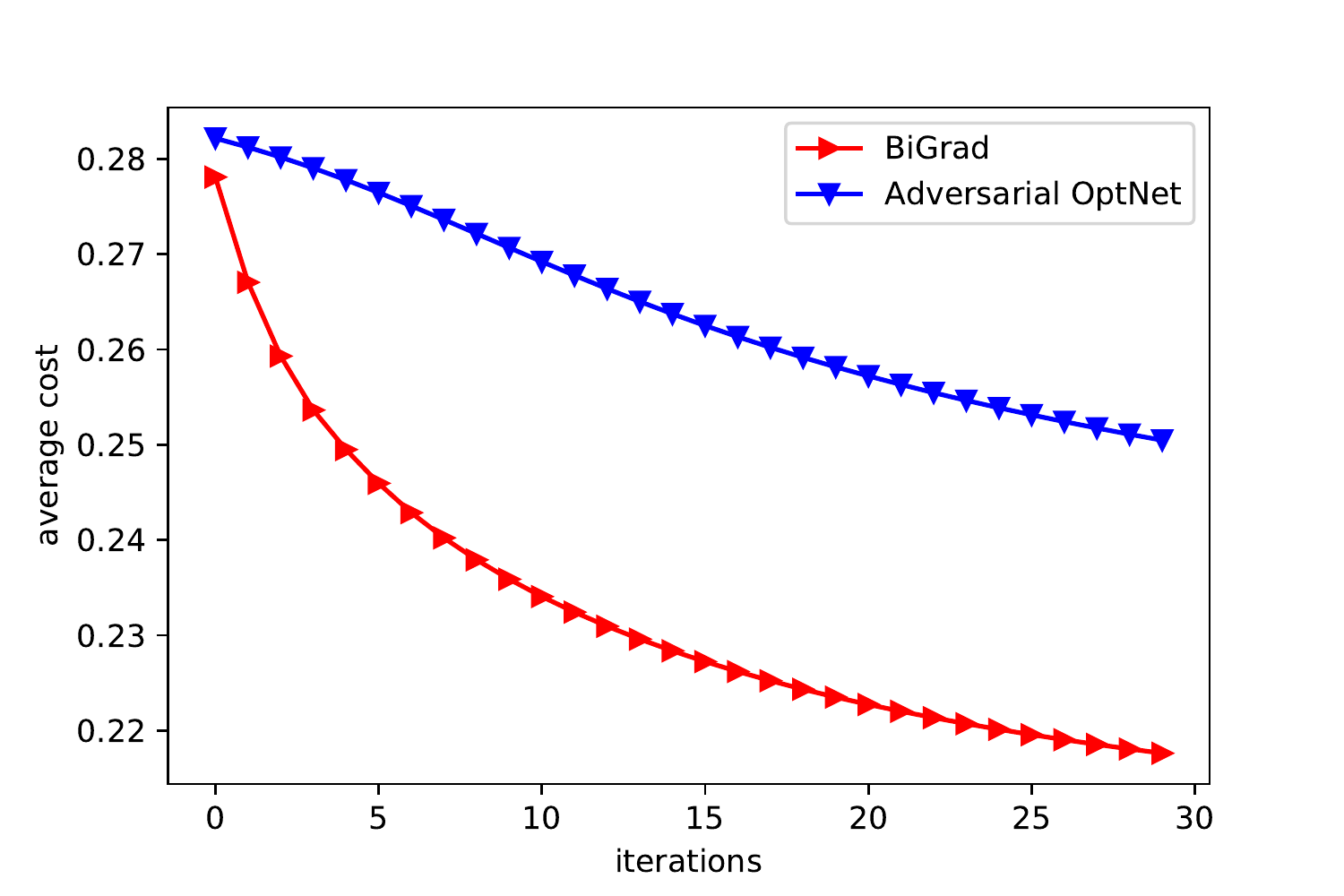}
	}
\caption{\footnotesize
	(a) Visualization of the Optimal Control Learning network, where a disturbance $\epsilon_t$ is injected based on the control signal $u_t$.
	(b) Comparison of the training performance for $N=2$, $T=20$ and epochs=$10$ of the \bigrad and the Adversarial version of the OptNet \citep{amos2017optnet}.}
\label{fig:optimal_control}
\vspace{-.3cm}
\end{figure}
To solve this problem we use Approximate Dynamic Programming (ADP) \citep{wang2010fast} that solves a proxy quadratic problem
\begin{align}\label{eq:optimal_control_ctrl}
\min_{u_t} ~~ & u_t^T P u_t + x_t Q u_t + q^t u_t ~~&
\text{s.t.} ~~ & \| u_t \|_2 \le 1
\end{align}
We can use the optimization layer as shown in Fig.\ref{fig:optimal_control}(a) and update the problem variables (e.g. $P,Q,q$) using gradient descent. We use the linear quadratic regulator (LQR) solution as the initial solution \citep{kalman1964linear}. The optimization module is replicated for each time step $t$, similarly to the Recursive Neural Network (RNN).
\begin{table}
\centering
	\caption{\footnotesize Optimal Control Average Cost; Bilevel approach improves (lower cost) over the two-step approach because is able to better capture the interaction between noise and control dynamics.}
	\label{tab:OptimalControl}
\footnotesize
\begin{tabular}{llll}
\toprule
& LQR & OptNet & Bilevel \\
\midrule
Adversarial & 2.736 & 0.2722 &  {\bf 0.2379 } \\
(10 steps) &   &   &    \\
(30 steps) & -  &  0.2511 &  {\bf 0.2181}  \\
 \bottomrule
\end{tabular}
\end{table}

We can build a resilient version of the controller in the hypothesis that an adversarial is able to inject a noise of limited energy, but is arbitrary dependent on the control $u$, by solving the following bilevel optimization problem
\begin{subequations}\label{eq:optimal_control_bilevel}
	\begin{align}
		\max _\epsilon ~~ &  Q(u_t,x_t+\epsilon) ~ &
		\text{s.t.} ~~& ||\epsilon|| \le \sigma \\
		u_t (\epsilon) &= \arg \min_{u_t } Q(u_t,x_t) ~ &
		\text{s.t.} ~~& \| u_t \|_2 \le 1
	\end{align}
\end{subequations}
where $Q(u,x) = u^T P u + x_t Q u + q^t u$ and we want to learn the parameters $z=(P,Q,q)$, where $y=u_t,x=\epsilon$ of Eq.\ref{eq:bilevel_continous}.


We evaluate the performance to verify the viability of the proposed approach and compare with LQR and OptNet \citep{amos2017optnet}, where the outer problem is substituted with the best response function that computes the adversarial noise based on the computed output; in this case, the adversarial noise is a scaled version of $Q u$ of Eq.\ref{eq:optimal_control_ctrl}.
Tab.\ref{tab:OptimalControl} and Fig.\ref{fig:optimal_control}(b) present the performance using \bigrad, LQR and the adversarial version of OptNet. 
\bigrad improves over two-step OptNet (Tab.\ref{tab:OptimalControl}), because is able to better model the interaction between noise and control dynamic.

\begin{table}
\centering
\footnotesize
\begin{tabular}{rllllllll}
\toprule
 $L_\infty \le \alpha$ &  DCNN &  Bi-DCNN &   CNN &   CNN*  \\
\midrule
   0 &          62.9   $\pm$ 0.3 &     {\bf 64.0}  $\pm$  0.4 &     63.4 $\pm$   0.7 &                 63.6 $\pm$   0.5 \\
   5 &          42.6   $\pm$ 1.0 &     {\bf 44.5}  $\pm$  0.2 &     43.8 $\pm$   1.2 &                 44.3 $\pm$   1.0 \\
  10 &          23.5   $\pm$ 1.5 &     {\bf 25.3}  $\pm$  0.8 &     24.3 $\pm$   1.0 &                 24.2 $\pm$   1.0 \\
  15 &          14.4   $\pm$ 1.4 &     {\bf 15.6}  $\pm$  0.7 &     14.6 $\pm$   0.7 &                 14.3 $\pm$   0.4 \\
  20 &           9.1   $\pm$ 1.2 &     {\bf 10.0}  $\pm$  0.6 &      9.2 $\pm$   0.4 &                  8.9 $\pm$   0.2 \\
  25 &           6.1   $\pm$ 1.0 &      {\bf 6.8}   $\pm$ 0.5 &      6.0 $\pm$   0.2 &                  5.9 $\pm$   0.2 \\
  30 &           3.9   $\pm$ 0.7 &      {\bf 4.4}  $\pm$  0.5 &      3.9 $\pm$   0.2 &                  3.9 $\pm$   0.1 \\
\bottomrule
\end{tabular}
	\caption{\footnotesize Performance on the adversarial attack with discrete features, with $Q=10$. DCNN is the single level discrete CNN, Bi-DCNN is the bilevel discrete CNN, CNN is the vanilla CNN, while  CNN* is the CNN where we add the bilevel discrete layer after vanilla training.}
	\label{tab:attack10}
	\vspace{-.6cm}
\end{table}

\begin{table*}[]
	\centering
	\footnotesize
	\begin{tabular}{rllllll}
		\toprule
		gradient &	 \multicolumn{2}{c}{accuracy  [12x12 maps]} &	  \multicolumn{2}{c}{accuracy  [18x18  maps]} &	 \multicolumn{2}{c}{accuracy  [24x24  maps] } \\
		type &  train  & {validation } & train  & {validation } &   train  & {validation  }  \\
		\midrule
		\bigrad(BB) &   {95.8} $\pm$ 0.2 &     {\bf94.5} $\pm$ 0.2 &  {\bf97.1} $\pm$ 0.0 &     {\bf96.4} $\pm$ 0.2 &   {98.0 }$\pm$ 0.0 &   {\bf97.8} $\pm$ 0.0 \\
		\bigrad(PT) &     91.7 $\pm$ 0.1 &     91.6 $\pm$ 0.1 &    94.3 $\pm$ 0.0 &     94.2 $\pm$ 0.1 &  95.7 $\pm$  0.0 &   95.6 $\pm$ 0.1 \\
		BB-1 &   95.9 $\pm$ 0.2 & 91.7 $\pm$ 0.1 &  96.7 $\pm$ 0.2 & 94.5 $\pm$ 0.1  &   97.1 $\pm$ 0.1 & 96.3 $\pm$ 0.2  \\
		PT-1 &  88.3 $\pm$ 0.2 & 87.5 $\pm$ 0.2 &  90.9 $\pm$ 0.4 & 90.6 $\pm$ 0.5 &   92.8 $\pm$ 0.1 & 92.8 $\pm$ 0.2  \\
		SL &   {\bf 100.0} $\pm$ 0.0 & 26.2 $\pm$ 2.4 &  {\bf 99.9} $\pm$ 0.1 & 20.2 $\pm$  0.5& {\bf99.1 }$\pm$ 0.2 & 14.0 $\pm$ 1.0 \\
		\bottomrule
	\end{tabular}
	\caption{\footnotesize Performance on the Dynamic Programming Problem with Interdiction. SL uses ResNet18.}
	\label{tab:SP}
\end{table*}

\subsection{Adversarial ML with discrete latent variables} \label{sec:Robust}
Machine learning models are heavily affected by the injection of intentional noise \citep{madry2017towards,goodfellow2014explaining}. An adversarial attack typically requires access to the machine learning model, in this way the attack model can be used during training to include its effect.  
Instead of training an end-to-end system as in \citep{goldblum2019adversarially}, where the attacker is aware of the model, we consider the case where the attacker can inject a noise at the feature level, as opposed to the input level (as in \citep{goldblum2019adversarially}), this allows us to model the interaction as a bilevel problem.
Thus, to demonstrate the use of a bilevel layer, we design a system that
is composed of a feature extraction layer, followed by a discretization layer that operates on the space of $\{0,1\}^m$, where $m$ is the hidden feature size, followed by a classification layer. The network used in the experiments is composed of two convolutional layers with max-pooling and two linear layers, all with relu activation functions, while the classification is a linear layer.
We consider a more limited attacker that is not aware of the loss function of the model and does not have access to the full model, but rather only to the input of the discrete layer
and is able to switch $Q$ discrete variables, 
The interaction of the discrete layer with the attacker is described by the following bilevel problem:
\begin{align} \label{eq:discretization_layer}
    \min_{ x \in Q} \max_{y  \in B} \langle z+x, y \rangle.
\end{align}
where $Q$ represents the sets of all possible attacks, $B$ is the budget of the discretization layer and $y$ is the output of the layer. 
For the simulation, we compute the solution by sorting the features by values and considering only the first B values, while the attacker will obscure (i.e. set to zero) the first $Q$ positions.
The output $y$ thus will have ones on the $Q$ to $B$ non-zero positions, and zero elsewhere. We train three models, on CIFAR-10 dataset for $50$ epochs. For comparison we consider:1) the vanilla CNN network (i.e. without the discrete features); 2) the network with the single-level problem (i.e. the single-level problem without attacker) and; 3) the network with the bilevel problem (i.e. the min-max discretization problem defined in Eq.\ref{eq:discretization_layer}).
We then test the networks to adversarial attack using the PGD \citep{madry2017towards} attack similar to \citep{goldblum2019adversarially}. Similar results apply for FGSM attack (Fast Gradient Sign Attack) \citep{goodfellow2014explaining}. We also tested the network trained as a vanilla network, where we added the min-max layer after training.
From the results (Tab.\ref{tab:attack10}), we notice: 1) The min-max network shows improved resilience to adversarial attack wrt to the vanilla network, but also with respect to the max (single-level) network; 2)	The min-max layer applied to the vanilla trained network is beneficial to adversarial attack; 3) The min-max network does not significantly change performance in presence of adversarial attack at the discrete layer (i.e. between Q=0 and Q=10). This example shows how bilevel layers can be successfully integrated into a Machine Learning system as differentiable layers.

\begin{table*}[!t]
	\centering
	\small
	\begin{tabular}{lrllrllrll}
		\toprule
		gradient & &  \multicolumn{2}{c}{accuracy}&  &  \multicolumn{2}{c}{accuracy} &  &  \multicolumn{2}{c}{accuracy} \\
		type &k  &  train  & {validation }   &  k &  train  & {validation }  &  k &  train  & {validation }  \\
		\midrule
		BB &  8 &   89.2 $\pm$  0.1 &     89.4  $\pm$  0.2 	&   10 &   91.9 $\pm$  0.1 &  {\bf 92.0} $\pm$  0.1 &   12 & 93.5 $\pm$  0.1 &     93.5 $\pm$ 0.2 \\
		PT  &  8 &   89.3 $\pm$  0.0 &     {\bf 89.4} $\pm$  0.1 &  10 &  92.0 $\pm$  0.0 &  91.9 $\pm$  0.1 &   12 &   {\bf 93.7} $\pm$  0.1 &     {\bf 93.7} $\pm$ 0.1 \\
		BB-1  & 8 & 84.0 $\pm$ 0.4 &  83.9 $\pm$ 0.4  	&   10 &   87.4  $\pm$ 0.3 &  87.5 $\pm$ 0.4 &   12 &    89.3 $\pm$ 0.1  &  89.3 $\pm$ 0.1  \\
		PT-1  & 8 &84.1 $\pm$ 0.4 &  84.1 $\pm$ 0.3 &   10 &   87.3 $\pm$ 0.3 &  87.0 $\pm$ 0.3  &   12 &    89.3 $\pm$ 0.0   &  89.5 $\pm$ 0.2  \\
		SL &  8 & {\bf94.2} $\pm$ 5.0  & 10.7 $\pm$ 3.9 &   10 &  {\bf 92.7} $\pm$ 5.4 &  9.4 $\pm$ 0.4  &   12 &  91.4 $\pm$ 2.3  &  9.3 $\pm$ 1.2  \\
		\bottomrule
	\end{tabular}
	\caption{\footnotesize Performance in terms of the accuracy of the TSP use case with interdiction. SL has higher accuracy during train but fails at test time. BB and PT are \bigrad variants.}
	\label{tab:TSP}
\end{table*}

\subsection{Dynamic Programming: Shortest path with Interdiction } \label{sec:SP}
We consider the problem of the Shortest Path with Interdiction, where the set of possible valid paths (see Fig.\ref{fig:SP_both}(a)) is $Y$ and the set of all possible interdiction is $X$. The mathematical problem can be written as
\begin{equation} \label{eq:SP}
\min_{y \in Y} \max_{x \in X} \langle z + x \odot w , y \rangle
\end{equation}
where $\odot$ is the element-wise product. This problem is multi-linear in the discrete variables $x,y,z$. 
\begin{figure}[!hbpt]
	\centering
	\subfigure[] {
	 \includegraphics[width=0.15\textwidth]{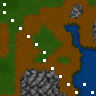}
	}
	\subfigure[] {
	\includegraphics[width=0.4\textwidth]{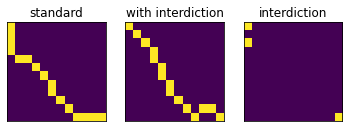}
	}
	\caption{ \footnotesize
		(a) Example Shortest Path in the Warcraft II tile set of \citep{guyomarchwarcraft}.
		(b) Example Shortest Path without (left) and with interdiction (middle). Even a small interdiction (right) has a large effect on the output.}
	\label{fig:SP_both}
	\vspace{-.6cm}
\end{figure}
The $z,w$ variables are the output of the neural network whose inputs are the Warcraft II tile images. The aim is to train the parameters of the weight network, such that we can solve the shortest path problem only based on the input image. For the experiments, we followed and adapted the scenario of \citep{poganvcic2019differentiation} and used the Warcraft II tile maps of \citep{guyomarchwarcraft}. 
We implemented the interdiction Game using a two-stage min-max-min algorithm \citep{kammerling2020oracle}. In Fig.\ref{fig:SP_both}(b) it is possible to see the effect of interdiction on the final solution. 
Tab.\ref{tab:SP} shows the performances of the proposed approaches, where we allow for $B=3$ interdictions and  we used tile size of $12 \times 12$, $18 \times 18$, $24 \times 24$. The loss function is the Hamming and $\ell_1$ loss evaluated on both the shortest path $y$ and the intervention $x$. 
The gradient estimated using Eq.\ref{eq:discrete_implicit_single_merged} (BB) provides more accurate results, at double of computation cost of PT. The single-level BB-1 approach outperforms PT, but shares similar computational complexity, while single-level PT-1 is inferior to PT. As expected, SL outperforms other methods during training, but completely fails during validation. Bigrad improves over single-level approaches because includes the interaction of the two problems.

\begin{figure}[!hbpt]
	\centering
	\subfigure[] {
		\includegraphics[width=0.25\textwidth, trim = 0 0 0 .1cm,clip]{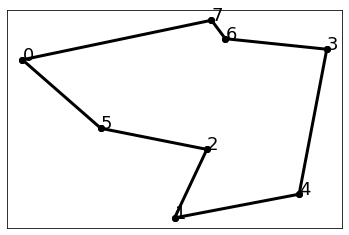}
	}
	\subfigure[] {
		\includegraphics[width=0.25\textwidth, trim = .1cm .1cm .1cm .1cm,clip]{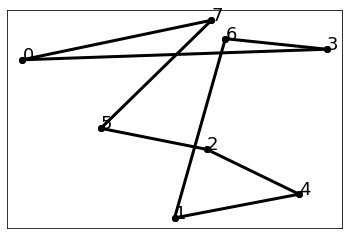}
	}
	\caption{\footnotesize Example of TSP with $8$ cities and the comparison of a TSP tour without (a) or with (b) a single interdiction. Even a single interdiction has a large effect on the final tour.}
	\label{fig:TSP}
	\vspace{-.3cm}
\end{figure}

\subsection{Combinatorial Optimization: Travel Salesman Problem (TSP) with Interdiction} \label{sec:TSP}
Travel Salesman Problem (TSP) with interdiction consists of finding the shortest route $y \in Y$ that touches all cities, where some connections $x \in X$ can be removed. 
The mathematical problem to solve is given by
\begin{equation} \label{eq:TSP}
	\min_{y \in Y} \max_{x \in X} \langle z + x \odot w , y \rangle
\end{equation}
where $z,w$ are cost matrices for the salesman and interceptor.
Similar to the dynamic programming experiment, we  implemented the interdiction Game using a two-stage min-max-min algorithm \citep{kammerling2020oracle}. Fig.\ref{fig:TSP} shows the effect of a single interdiction. 
The aim is to learn the weight matrices, trained with the interdicted solutions on a subset of the cities. 
Tab.\ref{tab:TSP} describes the performance in terms of accuracy on both shortest tour and intervention. We use Hamming and $\ell_1$ loss function. We only allow for $B=1$ intervention but considered $k = 8, 10$ and $12$ cities from a total of $100$ cities.  
Single and two-level approaches perform similarly in the training and validation. 
Since the number of interdiction is limited to one, the performance of the single-level approach is not catastrophic, while the supervised learning approach completely fails in the validation set. Bigrad thus improves over single-level and SL approaches. Since Bigrad(PT) has a similar performance of \bigrad(BB), thus PT is  preferable in this scenario, since it requires fewer computation resources. 

\section{Conclusions}
\bigrad generalizes existing single-level gradient estimation approaches and is able to incorporate Bilevel Programming as a learnable layer in modern machine learning frameworks, which allows to model of conflicting objectives as in adversarial attack. The proposed novel gradient estimators are also efficient and the proposed framework is widely applicable to both continuous and discrete problems. 
The impact of \bigrad has a marginal or similar cost with respect to the complexity of computing the solution of the Bilevel Programming problems. 
We show how \bigrad is able to learn complex logic when the cost functions are multi-linear. 






\section*{Ethical Statement and Limitations}
The present work does not have ethical implications, but share with all other machine learning approaches the potential to be used in a large multitude of applications; we expect our contribution to be used for the benefit and progress of our society. Our approach models bilevel problems with both discrete and continuous variables, but we have not explored the mixed integer programming approach, with mixed variables. We rely on the use of existing solvers to compute the current solution, thus we leave it to the next work to explore the potential to accelerate solving bilevel problems.   



    \bibliography{bilevel}

\clearpage
\appendix

\section{Supplementary Material; 
Implicit Bilevel Optimization: Differentiating through Bilevel Optimization Programming
}
\subsection{Extension for linear equalities and non-linear inequalities} \label{sec:linear_equality_and_nonlinear_inequality}
\subsubsection{Linear Equality constraints} \label{sec:linear_equality}
To extend the model of Eq.\ref{eq:bilevel_continous} to include linear equality constraints of the form $A x = b$ and $B y = c$ on the outer and inner problem variables, we use the following change of variables
\begin{align}
	x \to x_0 + A^ \perp x , ~~ & y  \to y_0 + B^ \perp y,
\end{align}
where $A^ \perp,B^ \perp$ are the orthogonal space of $A$ and $B$, i.e. $A A^ \perp = 0,B B^ \perp = 0$, and $x_0,y_0$ are one solution of the equations, i.e. $A x_0 = b, By_0=c$. 

\subsubsection{Non-linear Inequality constraints}\label{sec:nonlinear_inequality}
Similarly, to extend the model of Eq.\ref{eq:bilevel_continous} when we have non-linear inequality constraints, we use the barrier method approach \citep{boyd2004convex}, where the variable is penalized with a logarithmic function to violate the constraints. Specifically, let us consider the case where $f_i, g_i$ are inequality constraint functions, i.e. $f_i < 0, g_i < 0$, for the outer and inner problems. We then define new functions
\begin{align}
	f \to t f -\sum_{i=1}^{k_x} \ln (- f_i), ~~ & g \to t g -\sum_{i=1}^{k_y} \ln (- g_i).
\end{align}
where $t$ is a variable parameter, which depends on the violation of the constraints. The closer the solution is to violate the constraints, the larger the value of $t$ is. 

\subsection{Bilevel Cone programming} \label{sec:bilevel_cone} We show here how Theorem.\ref{th:bigrad_cont} can be applied to bi-level cone programming extending single-level cone programming results \citep{agrawal2019differentiating}, where we can use efficient solvers for cone programs to compute a solution of the bilevel problem \citep{ouattara2018duality}
\begin{subequations}\label{eq:bilevel_cone}
	\begin{align} 
		\min_{x} &~ c^Tx + (Cy)^T x \nonumber \\  
		 & ~ \text{s.t.} ~ Ax+z + R(y)(x-r) = b, ~ 
		s \in \mathcal{K} \\
		y \in & \arg  \min_{y } d^Ty + (Dx)^Ty \nonumber \\  
		& ~ \text{s.t.} ~  By+u + P(x) (y-p) = f, ~ 
		u \in \mathcal{K}
	\end{align}
\end{subequations}
In this bilevel cone programming, the inner and outer problem are both cone programs, where $R(y),P(x)$ represents a linear transformation, while $C,r,D,p$ are new parameters of the problem, while $\mathcal{K}$ is the conic domain of the variables. 
In the hypothesis that a local minima of Eq.\ref{eq:bilevel_cone} exists, we can use an interior point method to find such point. 
To compute the bilevel gradient, we then use the residual maps \citep{busseti2019solution} of the outer and inner problems.  
Indeed, we can then apply Theorem \ref{th:bigrad_cont}, where $F = N_1(x,Q,y)$ and $G = N_2(y,Q,x)$ are the normalized residual maps defined in \citep{busseti2019solution,agrawal2019differentiable} of the outer and inner problems. 


\subsection{Proofs}
\begin{proof}[Proof of Linear Equality constraints]
Here we show that 
\begin{align}
	x(u) = x_0 + A^ \perp u 
\end{align}
includes all solution of $Ax=b$. First we have that $A A^ \perp = 0$ and $Ax_0 = b$ by definition. This implies that $Ax(u) = A(x_0 + A^ \perp u) = Ax_0 = b$. Thus $\forall u \to Ax(u) = b$. 
The difference $x'-x_0$ belongs to the null space of $A$, indeed $A(x'-x_0) = Ax' - Ax_0 = b-b=0$. The null space of $A$ has size $n-\rho(A)$. If $\rho(A)=n$, where $A \in \R^{m \times n}, m \ge n$, then there is only one solution $x=x_0 = A^{\dagger}b$, $A^{\dagger}$ the pseudo inverse of $A$. If $\rho(A)<n$, then $\rho(A^ \perp)) = n - \rho(A)$ is a based of all vectors s.t. $Ax(u)=b$, since $\rho(A^ \perp)) = n - \rho(A)$ is the size of the null space of $A$. In fact $A^ \perp$ is the base for the null space of $A$. The same applies for $ y(v) = y_0 + B^ \perp v$ and $By(v) = c$.
\end{proof}

\begin{proof}[Proof of Theorem \ref{th:items}]
The second equation is derived by imposing the optimally condition on the inner problem. Since we do not have inequality and equality constraints we optimal solution shall equate the gradient w.r.t. $y$ to zero, thus $G=\nabla_y g = 0$. The first equation is also related to the optimality of the $x$ variable w.r.t. to the total derivative or hyper-gradient, thus we have that $0 = \dd_x f = \nabla_x f + \nabla_y f \nabla_x y$. In order to compute the variation of $y$, i.e. $\nabla_x y$ we apply the implicit theorem to the inner problem, i.e. $\nabla_x G + \nabla_y G \nabla_x y = 0$, thus obtaining $\nabla_x y = - \nabla^{-1}_y G \nabla_x G$.
\end{proof}

\begin{proof}[Proof of Theorem \ref{th:bigrad_cont}]
In order to prove the theorem, we use the Discrete Adjoin Method (DAM). 
Let consider a cost function or functional $L(x,y,z)$ evaluated at the output of our system. Our system is defined by the two equations $F=0, G=0$ from Theorem \ref{th:items}. Let us first consider the total variations: $\dd L, ~ \dd F =0 , ~ \dd G = 0$, where the last conditions are true by definition of the bilevel problem. When we expand the total variations, we obtain
\begin{eqnarray*}
\dd L &=& \nabla_x L \dd x + \nabla_y L \dd y + \nabla_z L \dd z \\
\dd F &=& \nabla_x F \dd x + \nabla_y F \dd y + \nabla_z F \dd z \\
\dd G &=& \nabla_x G \dd x + \nabla_y G \dd y + \nabla_z G \dd z 
\end{eqnarray*}
We now consider $\dd L + \dd F \lambda + \dd G \gamma  = [\nabla_x L + \nabla_x F \lambda  + \nabla_x G \gamma] \dd x + [\nabla_y L + \nabla_y F \lambda  + \nabla_y G \gamma ]\dd y + [\nabla_z L + \nabla_z F \lambda  + \nabla_z G \gamma ]\dd z$. We ask the first two terms to be zero to find the two free variables $\lambda,\gamma$: 
\begin{eqnarray}
\nabla_x L + \nabla_x F \lambda  + \nabla_x G \gamma &=& 0 \\
\nabla_y L + \nabla_y F \lambda  + \nabla_y G \gamma &=& 0
\end{eqnarray}
or in matrix form
$$
\begin{vmatrix}
		\nabla_x F & \nabla_x G\\
		\nabla_y F & \nabla_y F
\end{vmatrix} 
	\begin{vmatrix}
	\lambda \\
	 \gamma
\end{vmatrix} = -
	\begin{vmatrix}
	\nabla_x L  \\
	 \nabla_y L
\end{vmatrix}
$$
We can now compute the $\dd_z L = \nabla_z L + \nabla_z F \lambda  + \nabla_z G \gamma  $ with $\lambda, \gamma$ from the previous equation.
\end{proof}

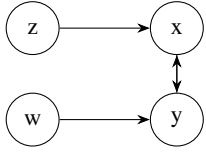
\begin{figure}[h]
   \centering
   \begin{tikzpicture}[
      mycircle/.style={
         circle,
         draw=black,
         fill=white,
         fill opacity = 0.3,
         text opacity=1,
         inner sep=0pt,
         minimum size=20pt,
         font=\small},
      myarrow/.style={-Stealth},
      node distance= .5cm and 1.2cm
      ]
      \node[mycircle] (z) {z};
      \node[mycircle,below =of z] (w) {w};
      \node[mycircle,right =of z] (x) {x};
      \node[mycircle,below =of x] (y) {y};
    \foreach \i/\j in {
      z/x/,
      x/y/,
      y/x/,
      w/y/
      }
      \draw [myarrow] (\i) -- node {} (\j);
    \end{tikzpicture} 
    \caption{Discrete Bilevel Variables: Dependence diagram}
    \label{fig:discrete_variables}    
\end{figure}

\begin{proof}[Proof of Theorem \ref{th:discrete}]
The partial derivatives are obtained by using the perturbed discrete minimization problems defined by Eqs.\ref{eq:discrete_basis}. We first notice that $\nabla_x \min_{y \in Y} \langle x,y \rangle = \arg  \min_{y \in Y} \langle x,y \rangle$. This result is obtained by the fact that $\min_{y \in Y} \langle x,y \rangle = \langle x,y^* \rangle$, where $y^* = \arg  \min_{y \in Y} \langle x,y \rangle $ and applying the gradient w.r.t. the continuous variable $x$; while Eqs. \ref{eq:discrete_perturbed} are the expected functions of the perturbed minimization problems. Thus, if we compute the gradient of the perturbed minimizer, we obtain the optimal solution, properly scaled by the inner product matrix. For example $\nabla_x \tilde{\Phi}_\eta = A x^*(z,y)$, with $A$ the inner product matrix. To compute the variation on the two-parameter variables, we have that 
$\dd L = \nabla_x L \dd x + \nabla_y L \dd y + \nabla_z L \dd z + \nabla_w L \dd w$ and that $\dd w/ \dd z = 0, \dd z/ \dd w = 0$ from the dependence diagram of Fig.\ref{fig:discrete_variables}
\end{proof}

\subsection{Gradient Estimation based on perturbation}
We can use the gradient estimator using the perturbation approach proposed in  \citep{berthet2020learning}. We thus have 
\begin{subequations}\label{eq:discrete_partial}
	\begin{align}
		\nabla_z x(z,y) &= A^{-1} \nabla_{z^2}^2 \tilde{\Phi}_\eta (z,y) \left.\right|_{\eta \to 0} \\
		\nabla_w y(w,z) &= C^{-1} \nabla_{w^2}^2 \tilde{\Psi}_\eta (w,z) \left.\right|_{\eta \to 0} \\
		\nabla_x y(x,w) &= D^{-1} \nabla_{x^2}^2 \tilde{\Theta}_\eta (x,w) \left.\right|_{\eta \to 0} \\
		\nabla_y x(z,y) &= B^{-1} \nabla_{y^2}^2 \tilde{W}_\eta (z,y) \left.\right|_{\eta \to 0} \\
		\nabla_z y &= \nabla_x y \nabla_z x
	\end{align}
\end{subequations}
and
\begin{subequations}\label{eq:discrete_perturbed}
	\begin{align}
		\tilde{\Phi}_\eta (z,y) &= \E_{u \sim U} \Phi (z + \eta u ,y)   \\
		\tilde{\Psi}_\eta (w,x) &=  \E_{u \sim U} \Psi (w + \eta u ,x) \\
		\tilde{\Theta}_\eta (x,w) &= \E_{u \sim U} \Psi (w  ,x + \eta u)  \\
		\tilde{W}_\eta (y,z) &= \E_{u \sim U} \Phi (z , y + \eta u )   
	\end{align}
\end{subequations}
, while
\begin{subequations}\label{eq:discrete_basis}
	\begin{align}
		\Phi (z,y) &=  \min_{x \in X} \langle z,x\rangle_A + \langle y,x\rangle_B \\
		\Psi (w,x) &=  \min_{y \in Y} \langle w,y\rangle_C + \langle x,y\rangle_D 
	\end{align}
\end{subequations}
which are valid under the conditions of \citep{berthet2020learning}, while $\tau$ and $\mu$ are hyper-parameters.

\subsection{Alternative derivation}\label{sec:alternative}
Let consider the problem $\min_{x\in K} \langle z,x \rangle_A$ and let us define $\Omega_x$ a penalty term that ensures $x \in K$. We can define the generalized lagragian $\mathbb{L}(z,x,\Omega) = \langle z,x \rangle_A + \Omega_x$. One example of $\Omega_x = \lambda^T|x-K(x)|$ or $\Omega_x = -\ln{|x-K(x)|}$ where $K(x)$ is the projection into $K$. To solve the Lagragian, we solve the unconstrained problem $\min_x \max_{\Omega_x} \mathbb{L}(z,x,\Omega_x)$. At the optimal point $\nabla_x \mathbb{L} = 0$. Let us define $F=\nabla_x \mathbb{L} = A^Tz+\Omega_x'$, then $\nabla_x F = \Omega_x''$ and $\nabla_z F = A^T$. 
If we have $F(x,z)=0$ and a cost function $L(x,z)$, we can compute $\dd_z L = \nabla_z L - \nabla_x L \nabla_x^{-1}F \nabla_z F$. 
Now $F(x,z,\Omega_x)=0$, we can apply the previous result and $\dd_z L = \nabla_z L -\nabla_x L \Omega_x''^{-1} A^T$. If we assume $\Omega_x'' = I$ and $\nabla_z L=0$, then $\dd_z L = - A \nabla_x L$.

\subsection{Memory Efficiency}
For continuous optimization programming, by separating the computation of the solution and the computation of the gradient around the current solution we 1) compute the gradient more efficiently, in particular, we compute second order gradient taking advantage of the vector-jacobian product (push-back operator) formulation without explicitly inverting and thus building the jacobian or hessian matrices; 2) use more advanced and not differentialble solution techniques to solve the bilevel optimization problem that would be difficult to integrate using automatic differentiable operations.
Using VJP we reduce memory use from $O(n^2)$ to $O(n)$. Indeed using an iterative solver, like generalized minimal residual method (GMRES) \citep{saad1986gmres}, we only need to evaluate the gradients of Eq.\ref{eq:bigrad_continuous} and not invert the matrix neither materialize the large matrix and computing matrix-vector products. Similarly, we use Conjugate Gradient (CG) method to compute Eq.\ref{eq:bilevel_continous_items}, which requires to only evaluating the gradient at the current solution and nor inverting neither materializing the Jacobian matrix. 
An implementation of a bilevel solver would have a memory complexity of $O(Tn)$, where $T$ is the number of iterations of the bilevel algorithm. 

\subsection{Experimental Setup and Computational Resources}
For the  Optimal Control with adversarial disturbance, we follow a similar setup of \citep{agrawal2019differentiable}, where we added the adversarial noise as described in the experiments. For the Combinatorial Optimization, we follow the setup of \citep{poganvcic2019differentiation}. The dataset is generated by solving the bilevel problem on the same data of \citep{poganvcic2019differentiation}. For section \ref{sec:SP}, we use the warcraft terrain tiles and generate optimal bilevel solution with the correct parameters $(z,w)$, where $z$ is the terrain transit cost and $w$ is the interdiction cost, considered constant to $1$ in our experiment. $X$ is the set of all feasible interdictions, in our experiment we allow the maximum number of interdictions to be $B$.
For section \ref{sec:TSP}, on the other hand the $z$ represents the true distances among cities and $w$ a matrix of the interdiction cost, both unknown to the model. $X$ is the set of all possible interdictions.
In these experiments, we solved the bilevel problem using the min-max-min algorithm \cite{kammerling2020oracle}.
For the Adversarial Attack, we used two convolutional layers with max-pooling, relu activation layer, followed by the discrete layer of size $m=2024$, $B=100$, $Q=0,10$. A final linear classification layer is used to classify CIFAR10. We run over $3$ runs, $50$ epochs, learning rate $lr=3e-4$ and Adam optimizer.
Experiments were conducted using a standard server with 8 CPU, 64Gb of RAM and GeForce RTX 2080 GPU with 6Gb of RAM. 

\subsection{Jacobian-Vector and Vector-Jacobian Products}
The Jacobian-Vector Product (JVP) is the operation that computes the directional derivative $J_f(x)u$, with direction $u \in \R^m$, of the multi-dimensional operator $f: \R^m \to \R^n$, with respect to $x \in \R^m$, where $J_f(x)$ is the Jacobian of $f$ evaluated at $x$. On the other hand, the Vector-Jacobian product (VJP) operation, with direction $v \in  \R^n$, computes the adjoint directional derivative $v^TJ_f(x)$. JVP and VJP are the essential ingredient for automatic differentiation \cite{elliott2018simple, baydin2018automatic}.

\end{document}